\documentclass[11pt,letterpaper,reqno]{amsart}

\usepackage[T1]{fontenc}
\usepackage{lmodern}
\usepackage{microtype}
\usepackage[margin=1.08in]{geometry}
\usepackage{amsmath,amssymb,mathtools}
\usepackage{booktabs,longtable,array}
\usepackage{enumitem}
\usepackage{xcolor}
\usepackage{hyperref}
\usepackage{xurl}

\hypersetup{
  colorlinks=true,
  linkcolor=blue!45!black,
  citecolor=blue!45!black,
  urlcolor=blue!55!black,
  pdfauthor={Kun LI; Li TIE; Peng WANG; Zihan LIU},
  pdftitle={Powers of pairwise differences: a proof of Colombo's 1928 determinant conjecture},
  pdfsubject={Determinants, apolarity, real Waring rank, distance-power matrices},
  pdfkeywords={pairwise-difference matrix, apolarity, binary form, real Waring rank}
}

\setlist[itemize]{leftmargin=2em,itemsep=0.25em,topsep=0.4em}
\setlist[enumerate]{leftmargin=2.2em,itemsep=0.25em,topsep=0.4em}

\allowdisplaybreaks
\numberwithin{equation}{section}

\newtheorem{theorem}{Theorem}[section]
\newtheorem{proposition}[theorem]{Proposition}
\newtheorem{lemma}[theorem]{Lemma}
\newtheorem{corollary}[theorem]{Corollary}
\newtheorem{remark}[theorem]{Remark}
\theoremstyle{definition}

\newtheorem{example}[theorem]{Example}

\newcommand{\RR}{\mathbb{R}}
\newcommand{\cP}{\mathcal{P}}
\newcommand{\rank}{\operatorname{rank}}
\newcommand{\Pf}{\operatorname{Pf}}

\newcommand{\LR}{L_{\RR}}
\newcommand{\defeq}{\mathrel{\vcentcolon=}}
\newcommand{\Mat}[1]{\mathbf{#1}}
\newcommand{\vct}[1]{\boldsymbol{#1}}

\title[Colombo's difference-power determinant]{An algebraic proof of Colombo's difference-power determinant conjecture}

\author{Kun LI\textsuperscript{*}}
\author{Li TIE}
\author{Peng WANG}
\author{Zihan LIU}

\date{26 August 2026}
\thanks{* Corresponding author.}
\thanks{E-mail addresses: \texttt{galoispure@163.com} (K. LI),
  \texttt{ttieli@icloud.com} (L. TIE),
  \texttt{wptest@qq.com} (P. WANG),
  \texttt{chelsealzh@163.com} (Z. LIU)}

\subjclass[2020]{Primary 15A15; Secondary 15A03, 15A18, 15B57, 13A50}
\keywords{pairwise-difference matrix, apolarity, binary form, real Waring rank}

\begin{document}

\begin{abstract}
Let \(n\ge2\) be even, let
\(\vct{\lambda}=(\lambda_1,\ldots,\lambda_n)\in\RR^n\) have pairwise
distinct coordinates, and define the difference-power matrix
\[
 \Mat{A}_d(\vct{\lambda})
 :=
 \bigl[(\lambda_r-\lambda_s)^d\bigr]_{r,s=1}^n,
 \qquad d\in\mathbb N.
\]
In 1928, Colombo proved that
\(\det\Mat{A}_{n-1}(\vct{\lambda})\ne0\)---and hence
\(\det\Mat{A}_{n-1}(\vct{\lambda})>0\)---and that
\(\rank\Mat{A}_d(\vct{\lambda})=d+1\) for \(0\le d<n-1\).
He conjectured that
\[
 \det\Mat{A}_d(\vct{\lambda})\ne0
 \qquad\text{for every } d\ge n-1.
\]
For even \(d\), the conjectured nonsingularity follows from previously
published results on distance-power matrices.  The remaining open cases
were therefore the supercritical odd exponents \(d\ge n+1\).  We prove
nonsingularity for all these odd exponents, thereby completing Colombo's
conjecture.  Consequently,
\[
 \rank\Mat{A}_d(\vct{\lambda})=\min\{n,d+1\}
 \qquad(d\in\mathbb N).
\]
Our proof converts a hypothetical kernel vector into a real binary form
having more projective real linear factors, counted with multiplicity,
than its real Waring length permits.
\end{abstract}

\maketitle
\tableofcontents

\section{Origin, historical progress, and main results}
\label{sec:introduction}

\subsection{Origin and formulation of Colombo's conjecture}

The conjecture considered here originates in Colombo's 1928 study of a
generalized Goursat problem \cite[pp.~34--36]{Colombo1928}.  Maclaurin
expansion leads, at each degree \(d\in\mathbb N\), to the coefficient
system
\begin{equation}\label{eq:colombo-system}
 \sum_{s=1}^n a_{sd}(\lambda_r-\lambda_s)^d=b_{rd},
 \qquad 1\le r\le n.
\end{equation}
Here \(\lambda_s\) is the parameter of the \(s\)-th characteristic line,
\(a_{sd}\) is an unknown Maclaurin coefficient, and \(b_{rd}\) is the
corresponding prescribed coefficient.

Abstracting the coefficient matrix, let \(n\ge2\), let
\(d\in\mathbb N\), and let
\(\vct{\lambda}=(\lambda_1,\ldots,\lambda_n)\in\RR^n\) have pairwise
distinct coordinates.  Define
\begin{equation}\label{eq:Ad-definition}
 \Mat{A}_d(\vct{\lambda})
 =\bigl[(\lambda_r-\lambda_s)^d\bigr]_{r,s=1}^n.
\end{equation}
For \(d=0\) we follow Colombo's convention \(0^0=1\), so that
\(\Mat{A}_0\) is the all-ones matrix.  For fixed \(d\),
\(\Mat{A}_d(\vct{\lambda})\) is precisely the coefficient matrix of
\eqref{eq:colombo-system}.

Colombo \cite[pp.~35--36, properties I--VIII]{Colombo1928} listed the
following properties, paraphrased in modern notation.
\begin{enumerate}[label=\textup{\Roman*.}]
\item The determinant is a symmetric homogeneous polynomial of total
degree \(nd\), of degree at most \(2d\) in each node.
\item The coefficient matrix is symmetric for even \(d\) and
skew-symmetric for odd \(d\):
\(\Mat{A}_d^{\mathsf T}=(-1)^d\Mat{A}_d\).
\item For odd \(d\), odd order forces the determinant to vanish, whereas
for even order the determinant is the square of a polynomial.
\item Within either parity class of exponents, unless the determinant
vanishes identically throughout that class, it vanishes for only finitely
many exponents.
\item The square of the Vandermonde product
\(\prod_{r<s}(\lambda_s-\lambda_r)^2\) divides the determinant.
\item At the threshold \(d=n-1\), the determinant is always nonzero.
\item For \(d<n-1\), the determinant vanishes identically and its matrix
rank is exactly \(d+1\).
\item The final item concerns an asymptotic limit in the exponent; it is
not used in the finite-\(d\) proof.
\end{enumerate}
Colombo \cite[p.~36]{Colombo1928} further reported verification for
\(n=4\) and \(n=6\), while his methods had not settled all even orders.
His remaining conjecture is that, for every even \(n\), pairwise distinct
real nodes, and integer \(d\ge n-1\), one has
\(\det\Mat{A}_d(\vct{\lambda})\ne0\).

\subsection{Existing even-exponent methods and the present result}

For positive even \(d\),
\[
 (\lambda_r-\lambda_s)^d=|\lambda_r-\lambda_s|^d,
\]
so \(\Mat{A}_d(\vct{\lambda})\) is a Hadamard power of a
one-dimensional Euclidean distance matrix.  Dyn--Goodman--Micchelli
\cite{DynGoodmanMicchelli1986} analyzed such positive powers, and Auer
\cite{Auer1998} classified invertible Hadamard powers of distance
matrices.  Their results give nonsingularity and the relevant inertia in
the even branch.  Bhatia--Jain
\cite[Theorem~1(iv)]{BhatiaJain2024} and Kapil--Mandeep--Singh
\cite{KapilMandeepSingh2024} later gave an equivalent viewpoint through
Kwong-matrix inertia and congruence with distance powers.  These methods
do not address the skew-symmetric higher odd-exponent branch.

A recent unreviewed entry in the ICM Conjectures catalogue records that
the citable literature located there left precisely the cases with even \(n\ge8\) and odd \(d\ge n+1\) unresolved
\cite{ICMConjecturesEntry}.  This identifies the remaining branch
addressed in the present paper.

For the higher odd exponents, a hypothetical kernel vector becomes a
real binary form.  The matrix equations force prescribed real linear
factors, whereas the Sylvester--Reznick bound limits their number by the
length of the same power-sum representation.  The two counts are
incompatible.  Combining this argument with Colombo's threshold result
and the published even-exponent classification covers the full
conjectural range.

\begin{theorem}[Colombo's determinant conjecture]\label{thm:main}
Let \(n\ge2\) be even, let
\(\vct{\lambda}=(\lambda_1,\ldots,\lambda_n)\in\RR^n\) have pairwise
distinct coordinates, and let \(d\ge n-1\) be an integer.  Then
\[
 \det \Mat{A}_d(\vct{\lambda})\ne0.
\]
\end{theorem}

\begin{corollary}[Complete rank formula]\label{cor:rank-formula}
Under the hypotheses on \(n\) and \(\vct{\lambda}\) above, for every
integer \(d\ge0\),
\[
 \rank \Mat{A}_d(\vct{\lambda})=\min\{n,d+1\}.
\]
\end{corollary}

\subsection{Outline of the proof}

Colombo's property VI supplies the threshold case \(d=n-1\), and his
property VII supplies the rank formula below the threshold
\cite[pp.~35--36]{Colombo1928}.  Section~\ref{sec:apolar} collects the
classical binary-form tools used in the remaining odd branch.  In
Section~\ref{sec:odd}, singularity gives a nonzero form, where \(X\) and
\(Y\) are commuting indeterminates,
\(F=\sum_jc_j(X+\lambda_jY)^d\) with at least \(n+1\) real projective
linear factors, while the Sylvester--Reznick inequality gives at most
\(n\); this contradiction proves every supercritical odd exponent.
Section~\ref{sec:completion} invokes Dyn--Goodman--Micchelli and Auer for
the even exponents, completes Theorem~\ref{thm:main} and
Corollary~\ref{cor:rank-formula}, and records the determinant sign as a
further result.  Appendix~\ref{app:lean-map} records the correspondence
between the odd-exponent derivation and its Lean formalization.

\section{Binary forms, apolarity, and real Waring rank}
\label{sec:apolar}

Matrices are denoted by bold uppercase letters, vectors by bold lowercase
or Greek letters, and scalar coordinates by ordinary italic letters.
As above, \(X\) and \(Y\) are commuting indeterminates.  Let
\begin{equation}\label{eq:binary-form-space}
 \cP_d=\RR[X,Y]_d
 =\left\{\sum_{a=0}^d f_aX^{d-a}Y^a:f_a\in\RR\right\}.
\end{equation}
The usual monomial basis of \(\cP_d\) is
\[
 X^d,X^{d-1}Y,\ldots,Y^d.
\]
For the apolar and Vandermonde calculations, we instead use the
binomially normalized monomial basis
\begin{equation}\label{eq:normalized-basis}
 \binom d0X^d,\binom d1X^{d-1}Y,\ldots,\binom ddY^d.
\end{equation}
Thus every \(F\in\cP_d\) has unique normalized coordinates
\[
 F(X,Y)=\sum_{a=0}^d\binom da f_aX^{d-a}Y^a,
\]
and \((f_0,\ldots,f_d)\) is its coordinate vector in
\eqref{eq:normalized-basis}.  For example,
\(X^2-Y^2\in\cP_2\), whereas \(X^2+Y\) is not homogeneous and belongs
to no \(\cP_d\).

A nonzero real linear form is \(L=\alpha X+\beta Y\) with
\((\alpha,\beta)\ne(0,0)\).  Two such forms determine the same real
projective factor if they differ by a nonzero scalar.  If
\(F=\prod_jL_j^{m_j}G\), where the \(L_j\) are pairwise
nonproportional and \(G\) has no real linear factor, define
\begin{equation}\label{eq:tau-definition}
 \tau(F)=\sum_jm_j.
\end{equation}
Thus \(\tau(F)\) counts real projective linear factors with
multiplicity.  The real Waring length of nonzero \(F\in\cP_d\) is
\begin{equation}\label{eq:real-waring-length}
 \LR(F)=\min\left\{
 r:F=\sum_{j=1}^r c_j(\alpha_jX+\beta_jY)^d,
 \ c_j\in\RR\setminus\{0\},\ (\alpha_j,\beta_j)\ne(0,0)
 \right\}.
\end{equation}
These conventions agree with Reznick \cite[\S1]{Reznick2013} and
Tokcan \cite[Introduction]{Tokcan2017}.

\begin{example}[Factor count and Waring length]\label{ex:tau-length}
For \(F=X^2-Y^2=(X-Y)(X+Y)\), one has \(\tau(F)=2\) and
\(\LR(F)=2\): the displayed difference of squares is a length-two
representation, while \(F\) is not the square of one real linear form.
The two quantities need not agree.  For \(G=X^2+Y^2\),
\(\tau(G)=0\) because \(G\) has no real projective zero, whereas
\(\LR(G)=2\) because \(G=X^2+Y^2\) and \(G\) is not a real square.
\end{example}

If \(p\in\RR[x]\) has degree at most \(d\), its degree-\(d\)
homogenization is
\begin{equation}\label{eq:homogenization}
 p^{[d]}(X,Y)=Y^dp(X/Y)\in\cP_d,
\end{equation}
expanded so that no division by \(Y\) remains.  We write
\(\tau(p)=\tau(p^{[d]})\) when \(d\) is fixed by context.  If
\(\deg p<d\), the homogenization has \(d-\deg p\) factors \(Y\), which
represent the projective root at infinity.
For example, if \(p(x)=x^2-1\) and \(d=3\), then
\[
 p^{[3]}(X,Y)=X^2Y-Y^3=Y(X-Y)(X+Y),
 \qquad \tau(p)=3.
\]
Here the factor \(Y\) is precisely the projective root at infinity.
The same example also has real Waring length three, since
\[
 p^{[3]}(X,Y)=\frac16(X+Y)^3-\frac16(X-Y)^3-\frac43Y^3.
\]
The displayed representation gives
\(\LR\bigl(p^{[3]}\bigr)\le3\), while
Theorem~\ref{thm:reznick} gives
\(3=\tau\bigl(p^{[3]}\bigr)\le\LR\bigl(p^{[3]}\bigr)\).  Hence
\(\LR\bigl(p^{[3]}\bigr)=3\).

\subsection{The apolar pairing and evaluation}

For
\[
 P=\sum_{a=0}^d\binom da p_aX^{d-a}Y^a,
 \qquad
 Q=\sum_{a=0}^d\binom da q_aX^{d-a}Y^a,
\]
define the normalized top-transvectant (apolar) pairing by
\begin{equation}\label{eq:apolar-pairing}
 \langle P,Q\rangle_d
 =\sum_{a=0}^d(-1)^a\binom da p_aq_{d-a}.
\end{equation}
This is the top transvectant in normalized coordinates; see Olver
\cite[Chapter~5]{Olver1999} and Ehrenborg--Rota
\cite{EhrenborgRota1993}.  Dolgachev--Kanev
\cite[Lemma~1.6]{DolgachevKanev1993} give the general pure-power
evaluation property.  Swapping the arguments in
\eqref{eq:apolar-pairing} shows that the pairing is symmetric for even
\(d\) and alternating for odd \(d\).

For \(t\in\RR\), define
\[
 L_t(X,Y)\defeq X+tY\in\cP_1,
 \qquad
 L_t^d(X,Y)\defeq(X+tY)^d\in\cP_d.
\]
In the normalized basis \eqref{eq:normalized-basis}, the coordinate
vector of \(L_t^d\) is \((1,t,\ldots,t^d)\).

\begin{lemma}[Pure powers and evaluation]\label{lem:apolar-identities}
For all \(u,v,t\in\RR\) and all \(F\in\cP_d\),
\begin{align}
 \langle L_u^d,L_v^d\rangle_d
 &= (v-u)^d, \label{eq:pure-power-pairing}\\
 \langle L_t^d,F\rangle_d
 &= F(-t,1). \label{eq:evaluation}
\end{align}
\end{lemma}

\begin{proof}
The normalized coefficient vector of \(L_u^d\) is
\((1,u,\ldots,u^d)\).  Hence the binomial theorem gives
\[
 \langle L_u^d,L_v^d\rangle_d
 =\sum_{a=0}^d(-1)^a\binom da u^av^{d-a}=(v-u)^d.
\]
If \(F=\sum_b\binom db f_bX^{d-b}Y^b\), then
\[
 \langle L_t^d,F\rangle_d
 =\sum_{a=0}^d(-1)^a\binom da t^af_{d-a}
 =\sum_{b=0}^d\binom db f_b(-t)^{d-b}=F(-t,1).
\]
\end{proof}

\subsection{Vandermonde independence of pure powers}

\begin{lemma}[Classical Vandermonde independence]\label{lem:pure-power-independence}
If \(m\le d+1\) and \(t_1,\ldots,t_m\) are pairwise distinct real
numbers, then \(L_{t_1}^d,\ldots,L_{t_m}^d\) are linearly independent in
\(\cP_d\).
\end{lemma}

\begin{proof}
In the normalized monomial basis, the first \(m\) coefficient coordinates
form the Vandermonde matrix.  Its classical determinant formula
\cite{MaconSpitzbart1958} is
\[
 \det\bigl[t_j^a\bigr]_{
 \substack{0\le a\le m-1\\1\le j\le m}}
 =\prod_{1\le i<j\le m}(t_j-t_i)\ne0,
\]
where the final inequality follows from pairwise distinctness.
\end{proof}

\subsection{Real factors and real Waring rank}

The external root-count result used below is the following homogeneous
form of Sylvester's rule of signs.

\begin{theorem}[Sylvester--Reznick]\label{thm:reznick}
Let \(F\in\RR[X,Y]_d\) be nonzero and not a \(d\)-th power of a real
linear form.  Then
\[
 \tau(F)\le \LR(F).
\]
\end{theorem}

\begin{proof}[Source and exact hypothesis]
Reznick \cite[Theorem~3.2]{Reznick2013} proves the sharper inequality
\(\tau(F)\le\sigma\le r\) for an honest real representation by
\(r\ge2\) distinct projective directions, where \(\sigma\) is a cyclic
sign-variation count.  Applying it to a minimal real Waring
representation gives the stated form.  Tokcan
\cite[Introduction, pp.~1--2]{Tokcan2017} records the same root-counting
consequence.
\end{proof}

\begin{remark}[Multiplicity matters]\label{rem:multiplicity}
The quantity \(\tau(F)\) counts factors with multiplicity.  Thus an
additional factor increases \(\tau(F)\) even if it coincides with one of
the already prescribed factors.  For example,
\((X-Y)^2(X+Y)\) has only two distinct real projective directions but
has \(\tau=3\), because \(X-Y\) occurs twice.
\end{remark}

\section{Odd exponents: the apolar--Waring contradiction}
\label{sec:odd}

Throughout this section, \(n\) is even, \(d>n-1\) is odd, and the nodes
are pairwise distinct.

\begin{proposition}\label{prop:odd-nonsingular}
The matrix \(\Mat{A}_d(\vct{\lambda})\) is nonsingular.
\end{proposition}

\begin{proof}
Suppose to the contrary that \(\Mat{A}_d(\vct{\lambda})\) is singular,
and choose a nonzero kernel vector
\(\vct{c}=(c_1,\ldots,c_n)^{\mathsf T}\).  Define
\begin{equation}\label{eq:F-kernel}
 F=\sum_{j=1}^n c_jL_{\lambda_j}^d\in\cP_d.
\end{equation}
Since \(n\le d+1\), Lemma~\ref{lem:pure-power-independence} gives
\(F\ne0\).  For every \(r\), Lemma~\ref{lem:apolar-identities} and the
kernel equation give
\[
 F(-\lambda_r,1)
 =\langle L_{\lambda_r}^d,F\rangle_d
 =\sum_{j=1}^n c_j(\lambda_j-\lambda_r)^d
 =-\bigl(\Mat{A}_d\vct{c}\bigr)_r=0.
\]
Consequently,
\[
 X+\lambda_rY\mid F
 \qquad(1\le r\le n).
\]
Indeed, after the linear change \(U=X+\lambda_rY\), \(V=Y\), the value
at \(U=0\) is \(V^dF(-\lambda_r,1)=0\), so every term contains \(U\).

The factors are pairwise nonproportional.  Hence
\begin{equation}\label{eq:Q-divides-F}
 Q(X,Y)\defeq\prod_{r=1}^n(X+\lambda_rY)
 \quad\text{divides}\quad F(X,Y).
\end{equation}
There is therefore a homogeneous \(H\in\RR[X,Y]_{d-n}\) such that
\(F=QH\).  Since \(d>n-1\) and \(d,n\) have opposite parity,
\(d-n\) is a positive odd integer.  Every nonzero real binary form of
odd degree has a real projective zero: on the unit circle,
\(H(-\vct v)=-H(\vct v)\), so the intermediate value theorem along a
semicircle gives a zero.  Thus \(H\) has a real linear factor.
For the concrete example \(H=X^3+Y^3\), the projective zero
\([1:-1]\) corresponds to the factor \(X+Y\).

The factor \(Q\) supplies \(n\) pairwise distinct real linear factors,
and \(H\) supplies one more occurrence, even if it repeats one of them.
Remark~\ref{rem:multiplicity} gives
\begin{equation}\label{eq:tau-lower}
 \tau(F)\ge n+1.
\end{equation}
Because \(Q\) contains at least two nonproportional factors, \(F\) is
not a pure \(d\)-th power.  After zero coefficients are omitted from
\eqref{eq:F-kernel}, the same representation gives
\(\LR(F)\le |\operatorname{supp}(\vct c)|\le n\), where
\[
 \operatorname{supp}(\vct c)=\{j:c_j\ne0\}.
\]
Theorem~\ref{thm:reznick} now yields
\[
 n+1\le\tau(F)\le\LR(F)\le n,
\]
a contradiction.
\end{proof}

\section{Completion and further consequences}
\label{sec:completion}

\begin{proof}[Proof of Theorem~\ref{thm:main}]
If \(d=n-1\), Colombo's property VI gives nonsingularity
\cite[pp.~35--36]{Colombo1928}.  If \(d>n-1\) is odd, apply
Proposition~\ref{prop:odd-nonsingular}.  If \(d\) is even, then
\(\Mat{A}_d=[|\lambda_r-\lambda_s|^d]\) is a one-dimensional
distance-power matrix.  The results of Dyn--Goodman--Micchelli
\cite{DynGoodmanMicchelli1986} and Auer \cite{Auer1998} show that it is
nonsingular in this range.  These cases exhaust all integers
\(d\ge n-1\).
\end{proof}

\begin{proof}[Proof of Corollary~\ref{cor:rank-formula}]
For \(0\le d<n-1\), Colombo's property VII gives
\(\rank\Mat{A}_d=d+1\) \cite[pp.~35--36]{Colombo1928}.  For
\(d\ge n-1\), Theorem~\ref{thm:main} gives
\(\rank\Mat{A}_d=n\).
\end{proof}

\begin{proposition}[Sign of the determinant]\label{prop:det-sign}
Let \(n\ge2\) be even, let \(\vct{\lambda}\in\RR^n\) have pairwise
distinct coordinates, and let \(d\ge n-1\).  Then
\begin{equation}\label{eq:det-sign}
 \det \Mat{A}_d(\vct{\lambda})>0\quad(d\ \text{odd}),
 \qquad
 \operatorname{sgn}\det \Mat{A}_d(\vct{\lambda})=(-1)^{n/2}
 \quad(d\ \text{even}).
\end{equation}
\end{proposition}

\begin{proof}
For a real skew-symmetric matrix \(\Mat M\) of even order, its Pfaffian
\(\Pf(\Mat M)\) satisfies the standard identity
\[
 \det\Mat M=\Pf(\Mat M)^2;
\]
see Ishikawa--Wakayama \cite[\S2]{IshikawaWakayama1995}.  If \(d\) is
odd, \(\Mat{A}_d\) is skew-symmetric and nonsingular by
Theorem~\ref{thm:main}, so its determinant is strictly positive.  If
\(d\) is even, Dyn--Goodman--Micchelli
\cite{DynGoodmanMicchelli1986} and Auer \cite{Auer1998} give the
distance-power inertia and hence the second sign in
\eqref{eq:det-sign}.
\end{proof}

\appendix

\section{Correspondence with the Lean formalization}
\label{app:lean-map}

\paragraph{Formal verification.}
The odd-exponent derivation in Section~\ref{sec:odd} has been formalized
and kernel-checked in Lean 4.32.1 with mathlib v4.32.1.  The corresponding
formalization source, including code and detailed documentation, is hosted
at \url{https://github.com/ttieli/Colombo1928}. The relevant
source dependency is
\[
 \begin{gathered}
 \texttt{Basic}\longrightarrow\texttt{Vandermonde}
 \longrightarrow\texttt{Apolar}\longrightarrow\texttt{ProjectiveRoots}\\
 \longrightarrow\texttt{SylvesterReznick}\longrightarrow\texttt{OddBranch}.
 \end{gathered}
\]

\begin{longtable}{>{\raggedright\arraybackslash}p{0.18\textwidth}
                  >{\raggedright\arraybackslash}p{0.20\textwidth}
                  >{\raggedright\arraybackslash}p{0.31\textwidth}
                  >{\raggedright\arraybackslash}p{0.20\textwidth}}
\caption{Correspondence between paper results and Lean proof terms.}
\label{tab:lean-map}\\
\toprule
Paper result & Lean file & Lean theorem & Formal role\\
\midrule
\endfirsthead
\toprule
Paper result & Lean file & Lean theorem & Formal role\\
\midrule
\endhead
\bottomrule
\endfoot
Threshold exponent &
\texttt{OddBranch} &
\path{threshold_nonsingular} &
Independently verifies Colombo's threshold result.\\
Odd root-count bound &
\texttt{OddBranch} &
\path{projectiveRootCount_powerSum_le_support} &
Formalizes the Sylvester--Reznick support bound.\\
All odd exponents &
\texttt{OddBranch} &
\path{odd_nonsingular} &
Matches the argument of Section~\ref{sec:odd}.\\
\end{longtable}

The formalization contains no project-specific axiom and no
\texttt{sorry} or \texttt{admit} placeholder.  Its axiom audit reports
only \texttt{propext}, \texttt{Classical.choice}, and
\texttt{Quot.sound}.  Kernel acceptance establishes closure of the
formal proof term; it does not replace external peer review or determine
literature priority.

\section*{Declaration on the use of generative AI and AI-assisted technologies}

During the research and preparation of this manuscript, the authors used
generative-AI systems, including large language models, under human
direction to assist with mathematical exploration, the development of
alternative proof routes, Lean~4 proof engineering and debugging,
literature search and organization, and bilingual drafting and LaTeX
editing.  These systems are not authors and bear no responsibility for the
work.  The authors determined the final mathematical statements and
exposition, checked the cited sources and the computational and formal
outputs, and take full responsibility for the accuracy, originality, and
integrity of the manuscript.

\section*{Acknowledgements}

The authors thank the authors of the cited works and the Lean/mathlib
community for the mathematical and formal infrastructure used here.

\bibliographystyle{unsrt}
\bibliography{references}

\end{document}